\documentclass[11pt]{article}

\usepackage[preprint]{acl}

\usepackage{times}
\usepackage{latexsym}

\usepackage[T1]{fontenc}

\usepackage[utf8]{inputenc}

\usepackage{microtype}

\usepackage{inconsolata}

\usepackage{graphicx}
\usepackage{subfigure}
\usepackage{booktabs} 
\usepackage[table]{xcolor}
\usepackage{algorithm}
\usepackage{algorithmic}
\usepackage{subcaption,amsfonts,dcolumn}
\usepackage{etoc}
\etocdepthtag.toc{mtchapter}
\etocsettagdepth{mtchapter}{subsubsection}
\etocsettagdepth{mtappendix}{none}
\usepackage{multirow}
\usepackage{amsmath}
\usepackage{amssymb}
\usepackage{mathtools}
\usepackage{amsthm}
\usepackage{enumitem}
\usepackage[capitalize,noabbrev]{cleveref}

\newtheorem{thm}{Theorem}
\newtheorem{prop}{Proposition}

\newtheorem{remark}{Remark}

\usepackage[textsize=tiny]{todonotes}

\usepackage{xspace}

\newcommand{\methodname}{Latent-Condensed Attention\xspace}
\newcommand{\sexyname}{LCA\xspace}

\def\ie{\mbox{\textit{i.e.}}}

\definecolor{LightBlue}{rgb}{0.925,0.957,1}

\def\mytitle{Latent-Condensed Transformer for Efficient Long Context Modeling}

\def\thmrevise{\textcolor{black}}
\def\revised{\textcolor{black}}

\title{\mytitle}

\author{Zeng You$^{1,2}$\thanks{Equal Contribution.}, Yaofo Chen$^{1}$\footnotemark[1], Qiuwu Chen$^{3}$\footnotemark[1], Ying Sun$^{3}$\footnotemark[1], \\
\textbf{Shuhai Zhang$^{1}$, Yingjian Li$^{2}$, Yaowei Wang$^{2,4}$\footnotemark[2], Mingkui Tan$^{1,5}\thanks{Corresponding Authors}$} \\ 
  $^1$South China University of Technology \\
  $^2$Pengcheng Laboratory \\
  $^3$AIGCode \\
  $^4$Harbin Institute of Technology \\
  $^5$Pazhou Laboratory
  }

\begin{document}
\maketitle
\begin{abstract}
Large language models (LLMs) face significant challenges in processing long contexts due to the linear growth of the key-value (KV) cache and quadratic complexity of self-attention.
Existing approaches address these bottlenecks separately: Multi-head Latent Attention (MLA) reduces the KV cache by projecting tokens into a low-dimensional latent space, while sparse attention reduces computation.
However, sparse methods cannot operate natively on MLA's compressed latent structure, missing opportunities for joint optimization.
In this paper, we propose \methodname (\sexyname), which directly condenses context within MLA's latent space, where the representation is disentangled into semantic latent vectors and positional keys. 
\sexyname separately aggregates semantic vectors via query-aware pooling and preserves positional keys via anchor selection. This approach jointly reduces both computational cost and KV cache without adding parameters. 
\revised{Beyond MLA, \sexyname's design is architecture-agnostic and readily extends to other attention mechanisms such as GQA.}
Theoretically, we prove a length-independent error bound. Experiments show \sexyname achieves up to \textbf{2.5$\times$} prefilling speedup and \textbf{90\%} KV cache reduction at 128K context while maintaining competitive performance.
\end{abstract}

\section{Introduction}
Efficient long-context modeling in large language models (LLMs) is essential for applications spanning full-document comprehension and extended multi-turn dialogues~\cite{openai2023gpt4,meta2024llama3,guo2025deepseekr1}.
However, transformer-based LLMs face two challenges: 1) the linear growth of key-value (KV) cache during decoding and 2) the quadratic computational complexity of self-attention~\cite{vaswani2017attention}, together hindering long‑context deployment.

To alleviate the memory overhead (challenge~1), Multi-head Latent Attention (MLA)~\cite{deepseekai2024deepseekv2} projects tokens into a shared low-dimensional latent space, significantly reducing the per-token KV cache size. During inference, MLA caches only the compressed latent vectors $\mathbf{C}^{KV}$ and the positional keys $\mathbf{K}^{R}$, where $\mathbf{K}^{R}$ encodes precise positional information via Rotary Position Embedding (RoPE)~\cite{su2024roformer}. Owing to its favorable trade-off between memory efficiency and model capacity, MLA has been widely adopted in recent long-context models~\cite{deepseekai2024deepseekv2,liu2024deepseekv3,guo2025deepseekr1,Team2025kimik2,kimiteam2025kimilinear}. Nevertheless, MLA still retains all $L$ latent vectors and performs dense attention, preserving the quadratic computational bottleneck. 

Separately, efficient attention methods mitigate quadratic computation (challenge 2) via sparsification. Early approaches used fixed sparse patterns~\cite{beltagy2020longformer,zaheer2020big}, which lack adaptability and risk discarding important information. Recent dynamic sparse methods~\cite{jiangminference,laiflexprefill,xu2025xattention} selectively skip computation blocks or evict less relevant tokens to reduce computation. However, these approaches remain susceptible to potentially important information loss.

Crucially, these two lines of work cannot be directly combined to address both challenges simultaneously. Most existing sparse methods operate on the computation of attention scores between original queries and keys, which require full-dimensional representations. 
When applied to MLA, these methods must first reconstruct full-dimensional KV matrices before performing sparsification. 
Consequently, these methods fail to leverage the compressed latent structure for further efficiency gains, particularly in reducing the KV cache beyond what MLA already achieves. This reveals a significant yet overlooked gap: existing methods lack a mechanism to perform efficient computation reduction natively within MLA's latent space.
 
In this paper, we bridge this gap by directly condensing redundant context within the MLA's latent space into a compact set of representatives, thereby reducing the number of vectors participating in attention computation and stored in the KV cache. However, MLA's disentangled representation presents a unique challenge: the \textit{semantic} component $\mathbf{C}^{KV}$ and the \textit{positional} component $\mathbf{K}^{R}$ have distinct functional properties due to their different encodings. Semantic information, which captures the continuous and often smooth content representations, can be aggregated across tokens without significant loss of fidelity~\cite{bolyatoken, feng2023efficient}; while positional encoding requires careful preservation to maintain accurate relative positional relationships. Thus, an effective condensation strategy must treat these two components differently.

Based on the above insights, we propose \textbf{\methodname} (\textbf{\sexyname}), an efficient attention mechanism that performs structured context condensation natively within latent space. \sexyname partitions context into groups and compresses each group into a representative latent vector via weighted pooling for the semantic component $\mathbf{C}^{KV}$, while preserving positional information by selecting the token with the highest relevance within each group for $\mathbf{K}^{R}$. In this way, \sexyname reduces both the KV cache size and the attention complexity, without introducing additional parameters.
We further provide a theoretical guarantee that the approximation error is uniformly bounded, independent of context length. \revised{Crucially, the decoupled condensation strategy is not specific to MLA; it generalizes to other attention mechanisms such as GQA, as we validate in Section~\ref{sec:generality}.} We summarize our contributions as follows:
\begin{itemize}[left=0pt]
    \item We propose \methodname (\sexyname), an efficient attention mechanism that performs structured context condensation directly in MLA's latent space, reducing both KV cache and attention computations, instead of sparsification only on reconstructed KV.
    \item  We design a latent-space condensation strategy that decouples semantic and positional processing: semantic information is adaptively aggregated via query-aware weighted pooling, while positional fidelity is preserved through hard anchor selection. This design avoids the signal-blending issue in conventional approaches.
    \item \sexyname can be easily integrated into pretrained models via lightweight fine-tuning, requiring no additional parameters. Experiments demonstrate that \sexyname delivers up to \textbf{2.5$\times$} prefill speedup and \textbf{90\%} KV cache reduction at 128K context length while maintaining comparable performance.
\end{itemize}

\section{Related Works}
\textbf{Efficient Attention}. To reduce the memory footprint of the KV cache, several attention variants have been proposed. Multi-Query Attention (MQA)~\cite{Shazeer2019mqa} shares a single key and value head across all query heads, while Grouped-Query Attention (GQA)~\cite{ainslie2023gqa} groups query heads to share key-value heads, improving efficiency while maintaining quality. More recently, Multi-head Latent Attention (MLA)~\cite{deepseekai2024deepseekv2} compresses KV states into a low-dimensional latent space, significantly reducing per-token cache size. It has been widely adopted in state-of-the-art long-context models~\cite{deepseekai2024deepseekv2,liu2024deepseekv3,guo2025deepseekr1,Team2025kimik2,kimiteam2025kimilinear}. However, these methods retain the quadratic computational complexity of standard attention.

Besides, sparse attention methods~\cite{zaheer2020big,beltagy2020longformer,zhang2025dga,chen2025core} reduce the computational cost by skipping some tokens or computation blocks. Early methods employ fixed sparse patterns~\cite{zaheer2020big,ding2023longnet,beltagy2020longformer,xiao2024efficient}, which lack adaptability. 
To introduce adaptability, DuoAttention~\cite{xiaoduoattention} learns head-specific static sparse patterns during an additional offline training phase. However, static patterns cannot adjust to the varying importance of tokens across different contexts. Recent dynamic sparse attention methods~\cite{laiflexprefill,jiangminference,xiaoduoattention} selectively compute only a subset of attention blocks. 

\revised{Unlike prior work that treats KV cache reduction and attention sparsity separately, \sexyname operates directly on MLA's compressed latent codes. By condensing information in the low-rank space, it avoids the overhead of full-dimensional reconstruction while simultaneously reducing both memory footprint and computational complexity.}

\noindent\textbf{System-Level Optimizations for Long-Context Inference}. Another line of work focuses on kernel-level optimizations and cache management for long-context inference.
FlashAttention~\cite{dao2022flashattention} and its successors speed up attention by optimizing memory access patterns.
PagedAttention~\cite{kwon2023pageattention} manages the KV cache more efficiently by allowing non-contiguous storage, reducing memory fragmentation.
Token-level eviction policies~\cite{li2024snapkv,hao2025omnikv,liu2025chunkkv} retain only the most salient tokens based on attention scores, while query-aware cache selection methods~\cite{tang2024quest} learn to select a compact set of KV pairs for each query.
These techniques are largely orthogonal to our approach.

\section{Preliminaries}\label{sec:preliminaries}
Multi-head Latent Attention (MLA)~\cite{deepseekai2024deepseekv2} is a core breakthrough for long-context large language models. It projects tokens into a low-dimensional latent space, drastically reducing per‑token KV‑cache memory while preserving modeling capability. Given an input sequence $\mathbf{X} \in \mathbb{R}^{L \times d}$, MLA computes compressed latents:
\begin{equation}
    \mathbf{C}^{Q} \small{=} \mathbf{X} W^{DQ}, \quad \mathbf{C}^{KV} \small{=} \mathbf{X} W^{DKV},
\end{equation}
where $W^{DQ}, W^{DKV} \small{\in} \mathbb{R}^{d \times d_c}$ are down-projection matrices and $d_c \small{\ll} d$. For each head $h$, queries, keys, and values are reconstructed from these latents:
\begin{align}
    \mathbf{Q}^h &\small{=} [\mathbf{C}^Q W^{UQ_h} , \mathbf{Q}^{R_h}] \in \mathbb{R}^{L \times d_k}, \\
    \mathbf{K}^h &\small{=} [\mathbf{C}^{KV} W^{UK_h}, \mathbf{K}^{R}] \in \mathbb{R}^{L \times d_k}, \\
    \mathbf{V}^h &\small{=} \mathbf{C}^{KV} W^{UV_h}\in \mathbb{R}^{L \times d_v},
\end{align}
where $\mathbf{Q}^{R_h}, \mathbf{K}^{R} \in \mathbb{R}^{L \times d_r}$ are Rotary Position Embedding (RoPE)~\cite{su2024roformer} augmented components that encode positional information. $W^{UQ_h}, W^{UK_h} \in \mathbb{R}^{d_c \times d_k'}$ and $W^{UV_h} \in \mathbb{R}^{d_c \times d_v}$ are the up-projection matrices for head $h$, with $d_k' + d_r = d_k$. The attention is computed as:
\begin{equation}
    \text{MLA}^h = \operatorname{softmax}\left( \frac{\mathbf{Q}^h (\mathbf{K}^h)^\top}{\sqrt{d_k}} \right) \mathbf{V}^h.
\end{equation}
The outputs of all heads are then combined via concatenation. During inference, MLA caches only $\mathbf{C}^{KV}$ and $\mathbf{K}^{R}$, reducing per-token memory. 
Owing to its favorable efficiency‑accuracy trade‑off, MLA has been widely adopted as a core component in state‑of‑the‑art long‑context LLMs~\cite{liu2024deepseekv3,deepseekai2025deepseekv32,Team2025kimik2,kimiteam2025kimilinear}.

\noindent\textbf{Limitations of MLA}. 
While MLA substantially reduces per-token KV-cache memory by projecting tokens into a low-dimensional latent space, it still requires all $L$ tokens to participate in attention computation, resulting in significant redundant computation and memory access and leaving the quadratic dependence on context length unchanged. 
Existing efficient attention methods~\cite{xiao2024efficient,xiaoduoattention,laiflexprefill} can alleviate computational cost by sparsifying attention, but they operate on reconstructed full-dimensional key–value states. When applying these methods to MLA, this reconstruction step negates the efficiency benefits of its low-dimensional latent representation. As a result, current approaches \textbf{fail to exploit MLA's latent structure to jointly reduce attention computation and KV-cache growth}.

\begin{figure*}
    \centering
    \includegraphics[width=\linewidth]{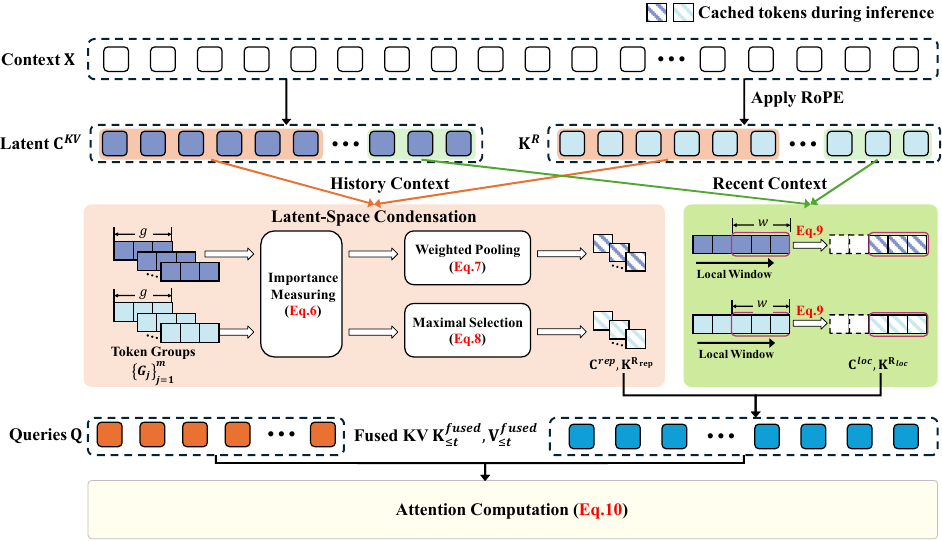}
    \caption{An overview of \methodname (\sexyname). The history context is condensed into a compact set of representatives via group-wise condensation in the latent space: semantic information is aggregated through weighted pooling of $\mathbf{C}^{KV}$, while positional information is preserved via anchor selection from $\mathbf{K}^{R}$. 
        The recent context is retained in full fidelity without condensation.}
    \label{fig:method}
\end{figure*}

\section{\methodname}
\subsection{Motivations and Method Overview}
Long context sequences often exhibit substantial redundancy: only a small subset of tokens is typically relevant for the task~\cite{chen2025core,zhang2025dga, xiao2024efficient}.
This suggests that computing attention over all $L$ tokens incurs unnecessary computational overhead. 
While Multi-head Latent Attention (MLA) effectively compresses the KV cache into a low-dimensional latent space, it still performs quadratic attention over all $L$ tokens.
We propose to address this bottleneck by condensing redundant context before attention computation. 
Crucially, unlike methods that sparsify attention on reconstructed full-dimensional states, \textit{we operate natively within compressed latent space, leveraging its latent structure for further efficiency gains}.

Unfortunately, MLA’s latent representation is naturally \textit{disentangled}: each token is encoded by a compressed semantic latent $\mathbf{C}^{KV}$ and a residual positional key $\mathbf{K}^R$ with precise positional information via RoPE~\cite{su2024roformer}.
The semantic and positional components differ functionally due to their distinct encodings: 
semantic information can often be aggregated~\cite{feng2023efficient,bolyatoken}, whereas positional encoding is highly nonlinear and should not be arbitrarily blended. Hence, an effective condensation strategy must treat these two components differently.

Motivated by these insights, we propose \textbf{\methodname (\sexyname)}, which employs a dual-path group-wise condensation strategy: semantic vectors are aggregated via query-aware pooling, while positional fidelity is preserved by selecting a positional anchor per group.
By operating natively in MLA’s latent space, \sexyname reduces both the number of cached vectors and the attention complexity, without introducing additional parameters. 
The framework of \sexyname is illustrated in Figure \ref{fig:method}, with its algorithm in Algorithm \ref{alg:method}.

\subsection{Latent-Space Condensation}\label{sec:latent-space-condensation}

We perform latent-space token condensation by explicitly distinguishing semantic and positional information, and applying different reduction operators to each. This enables effective reduction of long-range context while preserving sufficient information required for attention computation.

Given the latent matrix $\mathbf{C}^{KV} \in \mathbb{R}^{L \times d_c}$ and the positional keys $\mathbf{K}^R \in \mathbb{R}^{L \times d_r}$, we partition the history context into
$m = \left\lfloor \frac{L - w}{g} \right\rfloor$ contiguous groups $\{G_j\}_{j=1}^{m}$ of size $g$, where $w$ denotes a local window size that preserves fine-grained local context (detailed later).
To assess token relevance within each group, we adopt a query-aware scoring mechanism that reflects the current decoding focus rather than relying on static heuristics.
Following~\citet{laiflexprefill}, we compute a summary query by averaging the last $g$ queries, \ie, $\bar{\mathbf{q}} = \text{Average}(\mathbf{Q}_{[-g:]})$.
For each token $i \in G_j$, we compute an importance score
$s_i = \bar{\mathbf{q}}^{\top} \mathbf{k}_i / \sqrt{d_k}$,
followed by a group-wise softmax normalization:
\begin{equation}\label{eq:importance-score}
    \alpha_i^{(j)} = \frac{\exp(s_i)}{\sum_{k \in G_j} \exp(s_k)}, \quad \forall i \in G_j.
\end{equation}
We treat semantic and positional components differently based on their distinct functional properties.

\noindent\textbf{Semantic condensation via weighted pooling}. 
For the semantic latent vectors $\mathbf{c}_i^{KV}$ within each group, we compute a representative latent via weighted pooling:
\begin{equation}\label{eq:pooling}
    \mathbf{c}_j^{\text{rep}} = \sum_{i \in G_j} \alpha_i^{(j)} \mathbf{c}_i^{KV} \in \mathbb{R}^{d_c},
\end{equation}
where $\mathbf{c}_i^{KV}$ denotes the $i$-th row of $\mathbf{C}^{KV}$. 
Unlike sparse attention methods that discard tokens or computation blocks, weighted pooling preserves information from all tokens in the group while emphasizing those most relevant to the current query.
This design admits a simple yet principled interpretation.
The following proposition shows that weighted pooling yields the optimal representative under an expected reconstruction criterion.

\begin{prop}(\textbf{Optimal Condensation via Weighted Pooling})\label{lem:optimal-pooling}
Let \(\{\mathbf{c}_i \in \mathbb{R}^{d_c}\}_{i=1}^g\) be a set of latent vectors and \(\boldsymbol{\alpha} \in \Delta^{g-1}\) a probability distribution over them, with \(\alpha_i \ge 0\) and \(\sum_i \alpha_i = 1\).
Consider the expected squared reconstruction error
\[
\mathcal{L}(\mathbf{c}^{\text{rep}}) = \mathbb{E}_{i \sim \boldsymbol{\alpha}} \bigl[ \|\mathbf{c}_i - \mathbf{c}^{\text{rep}}\|_2^2 \bigr]
= \sum_{i=1}^g \alpha_i \|\mathbf{c}_i - \mathbf{c}^{\text{rep}}\|_2^2 .
\]
The vector \(\mathbf{c}^{\text{rep}}\) that minimizes \(\mathcal{L}\) is uniquely given by the convex combination \(\mathbf{c}^{\text{rep}} = \sum_{i=1}^g \alpha_i \mathbf{c}_i.\)
\end{prop}

Applying Proposition~\ref{lem:optimal-pooling} with $\alpha_i = \alpha_i^{(j)}$ shows that our weighted pooling minimizes the expected error in representing each group’s semantic content.
The full proof is provided in Appendix~\ref{app:proof-pooling}.

\noindent\textbf{Positional anchoring via max-selection}.
The residual positional keys $\mathbf{k}_i^R$ are encoded using Rotary Position Embedding (RoPE)~\cite{su2024roformer}, which is a nonlinear function of absolute position.
Direct pooling over $\mathbf{k}_i^R$ would blend distinct positional signals and distort relative positional relationships.
To preserve accurate positional information, we select the token with the highest importance score within each group as the positional anchor:
\begin{equation}\label{eq:k_select}
    I_j = \arg\max_{i \in G_j} \alpha_i^{(j)}, 
    \quad \mathbf{k}_j^{R_{\text{rep}}} = \mathbf{k}_{I_j}^R .
\end{equation}
This strategy ensures that each representative’s positional information corresponds to a concrete token position, preserving valid attention geometry.
In contrast to semantic aggregation, positional information is preserved through hard selection to avoid blending incompatible positional signals.

\noindent\textbf{Preserving Fine-Grained Local Context}. 
While latent condensation effectively reduces redundancy in long-range context, fine-grained local information remains critical for accurate next-token prediction~\cite{zong2021long,yang2021context}.
Accordingly, we retain a local window of the most recent $w$ latent vectors in full fidelity.
Any residual tokens that do not form a complete group are also included in this window.
Specifically, the local window consists of the preceding $k = w + r$ latent vectors, where $r = (L - w) \bmod g$, ensuring that every token is accounted for.

\begin{algorithm}[t]
\caption{\methodname}
\label{alg:method}
\begin{algorithmic}[1]
\REQUIRE Sequence length $L$, latents $\mathbf{C}^{KV} \small{=} [\mathbf{c}^{KV}_1;\ldots; \mathbf{c}^{KV}_L]$, residual keys $\mathbf{K}^{R} \small{=} [\mathbf{k}^R_1;\ldots;\mathbf{k}^R_L]$, queries $\mathbf{Q} \small{=} [\mathbf{q}_1;\ldots;\mathbf{q}_L]$, window size $w$, group size $g$.
\STATE $\bar{\mathbf{q}}\small{\gets} \text{Avg}([\mathbf{q}_{L-g+1};\cdots; \mathbf{q}_L])$, $m \small{\gets}\lfloor (L - w) / g \rfloor$\\
\FOR{$j = 1$ to $m$}
    \STATE $G_j\small{=}\{(j\small{-}1)g\small{+}1, (j\small{-}1)g\small{+}2,\dots, jg\}$
    \STATE $\alpha_i^{(j)} \gets \text{softmax}_i(\bar{\mathbf{q}}^\top \mathbf{k}_i / \sqrt{d_k})$ for $i \in G_j$
    \STATE $\mathbf{c}_j^{\text{rep}} \gets \sum_{i \in G_j} \alpha_i^{(j)} \mathbf{c}_i^{KV}$
    \STATE $I_j \gets \arg\max_{i \in G_j} \alpha_i^{(j)}$
    \STATE $\mathbf{k}_j^{\text{rep}} \gets [\mathbf{c}_j^{\text{rep}} W^{UK}, \mathbf{k}^R_{I_j}]$, $\mathbf{v}_j^{\text{rep}} \gets \mathbf{c}_j^{\text{rep}} W^{UV}$
\ENDFOR
\STATE $\mathbf{K}^{\text{rep}} \small{\gets} [\mathbf{k}_1^{\text{rep}};\dots;\mathbf{k}_{m}^{\text{rep}}]$, $\mathbf{V}^{\text{rep}} \small{\gets} [\mathbf{v}_1^{\text{rep}};\dots;\mathbf{v}_{m}^{\text{rep}}]$ 
\STATE $k \gets \max\left(w, w + \left((L-w) \bmod g)\right)\right)$ 
\STATE $\mathbf{K}^{\text{loc}} \small{=} [\mathbf{k}_{L-k+1}; \dots; \mathbf{k}_L], \mathbf{V}^{\text{loc}} \small{=} [\mathbf{v}_{L-k+1}; \dots; \mathbf{v}_L]$\\
\STATE $\mathbf{K}^{\text{fused}} \small{\gets} [\mathbf{K}^{\text{rep}}; \mathbf{K}^{\text{loc}}]$, $\mathbf{V}^{\text{fused}} \small{\gets} [\mathbf{V}^{\text{rep}}; \mathbf{V}^{\text{loc}}]$
\STATE $\mathbf{Attn} \gets \text{softmax}\bigl(\mathbf{Q} (\mathbf{K}^{\text{fused}})^\top / \sqrt{d_k}\bigr) \mathbf{V}^{\text{fused}}$

\ENSURE Attention output $\mathbf{Attn}_t$
\end{algorithmic}
\end{algorithm}

\subsection{Attention with Condensed Context}\label{sec: Attention with Condensed Context}

\textbf{Representative Key–Value Construction}. For each group $G_j$, the final representative key and value matrices are constructed using MLA's original up‑projection matrices:
\begin{equation*}
    \mathbf{k}_j^{\text{rep}} \small{=} \left[ \mathbf{c}_j^{\text{rep}} W^{UK},  \mathbf{k}_j^{R_{\text{rep}}} \right], \quad
    \mathbf{v}_j^{\text{rep}} \small{=} \mathbf{c}_j^{\text{rep}} W^{UV}.
\end{equation*}
Collecting all representatives, we obtain condensed key and value matrices:
\begin{equation*}
    \mathbf{K}^{\text{rep}} \small{=} [\mathbf{k}_1^{\text{rep}}; \dots; \mathbf{k}_m^{\text{rep}}],
\mathbf{V}^{\text{rep}} \small{=} [\mathbf{v}_1^{\text{rep}}; \dots; \mathbf{v}_m^{\text{rep}}].
\end{equation*}
We preserve the corresponding keys and values of the local context without condensation:
\begin{equation}
    \mathbf{K}^{loc} \small{=} [\mathbf{k}_{L-k+1}; \dots; \mathbf{k}_L], \mathbf{V}^{loc} \small{=} [\mathbf{v}_{L-k+1}; \dots; \mathbf{v}_L],
\end{equation}
where each \(\mathbf{k}_j = [\mathbf{c}_j^{KV} W^{UK}, \mathbf{k}_j^R]\) and \(\mathbf{v}_j = \mathbf{c}_j^{KV} W^{UV}\). This ensures that the most recent and relevant tokens are always attended to with full fidelity, seamlessly integrating condensed long-range context with detailed local information.

\noindent\textbf{Prefilling with Fused Context}. During prefilling, the condensed key and value matrices $\mathbf{K}^{\text{rep}}$ and $\mathbf{V}^{\text{rep}}$ are combined with the full-fidelity local keys and values $\mathbf{K}^{l}$ and $\mathbf{V}^{l}$ to form the fused context:
\begin{equation*}
    \mathbf{K}^{\text{fused}} = \bigl[ \mathbf{K}^{\text{rep}}; \mathbf{K}^{loc} \bigr],
\mathbf{V}^{\text{fused}} = \bigl[ \mathbf{V}^{\text{rep}}; \mathbf{V}^{loc} \bigr].
\end{equation*}
We then compute attention over this fused set:
\begin{equation}
    \mathbf{Attn} = \operatorname{softmax}\!\left( \frac{\mathbf{Q} (\mathbf{K}^{\text{fused}})^\top}{\sqrt{d_k}} \right) \mathbf{V}^{\text{fused}}.
\end{equation}
Notably, when the total sequence length $L$ is smaller than $w + g$, the long-range condensation provides negligible efficiency benefits.
In such cases, we revert to standard MLA computation, ensuring that the model always operates in the most efficient regime without sacrificing accuracy. 

\noindent\paragraph{Decoding and Online Cache Update.} 
During autoregressive decoding, we only cache only the $m$ representative latent vectors $\{\mathbf{c}_j^{\text{rep}}\}_{j=1}^m$ and their positional anchors $\{\mathbf{k}_j^{R_{\text{rep}}}\}_{j=1}^m$, together with the $w$ most recent latent vectors in full fidelity.
Once the number of newly generated tokens reaches the group size $g$, we use the average query of the last $g$ tokens to compute importance scores (as in Eq.~\eqref{eq:importance-score}) for the earliest $g$ tokens in the buffer. These $g$ tokens are then condensed into a single representative via Eqs.~\eqref{eq:pooling} and~\eqref{eq:k_select}. This new representative is appended to the cache.
This on‑the‑fly way ensures that the cache size remains $\mathcal{O}(m+w)$ during decoding, and each decoding step attends to only $m+w$ keys and values, thereby reducing both memory footprint and per‑token decoding latency.

\subsection{Further Analysis of \sexyname}
\textbf{Theoretical Analysis}. 
We provide a formal guarantee on the approximation quality of \sexyname. The following theorem bounds the error between the output of \sexyname and MLA attention.

\begin{thm}[\textbf{Uniform Error Bound}]
\label{thm:uniform-bound}
Fix a query $\mathbf{q}_t$ with $\|\mathbf{q}_t\|_2 \le Q$ and assume $\|\mathbf{v}_i\|_2 \le V$ for all values. For each distant-history group $G_j$ \thmrevise{with group size $g$}, let $\mathbf{k}_j^{\mathrm{rep}},\mathbf{v}_j^{\mathrm{rep}}$
be the representative key and value
{corresponding to the tokens in $G_j$},
and let $\delta_k,\delta_v$ be uniform bounds on the per-token key and value deviations: $\|\mathbf{k}_i - \mathbf{k}_j^{\mathrm{rep}}\|_2 \le \delta_k,
\|\mathbf{v}_i - \mathbf{v}_j^{\mathrm{rep}}\|_2 \le \delta_v.$
\thmrevise{Any remaining tokens are left uncompressed.}
Then the $\ell_2$ error between the full MLA attention and \sexyname satisfies
\[
\begin{aligned}
\big\|\mathbf{Attn}_t-\mathbf{Attn}_t^{\mathrm{LCA}}\big\|_2
&\le V\big(e^{2Q\delta_k/\sqrt{d_k}}-1\big)\\
&\quad+\delta_v
\thmrevise{+2V\,\frac{\sqrt{g}-1}{\sqrt{g}+1}}.
\end{aligned}
\]
\end{thm}
\thmrevise{For fixed $g$, the bound is independent of the context length $L$. Its final term accounts for the change in softmax normalization when each group of tokens is represented by a single token. This term increases with $g$, reflecting
the greater worst-case approximation error from larger groups.
This trend is consistent with the group-size sensitivity results
in Figure~\ref{fig:ablation_group_size} and Table~\ref{tab:group-size}.}
Moreover, the local window ensures the most recent tokens are never compressed ($\delta_k=\delta_v=0$ for those tokens). The full proof is provided in Appendix~\ref{app:thm_proof}. Experimentally, with practical group sizes, this approximation maintains performance comparable to MLA while delivering significant efficiency gains.

\revised{To empirically assess the key and value approximation
assumptions, we measure the per-token deviations $\delta_k$ and $\delta_v$ across various context lengths and tasks. As detailed in Appendix~\ref{app:kv_deviations}, the average relative deviation remains consistently below 5\% for keys and 4\% for values, confirming the practical soundness of our grouping assumption.}

\noindent\textbf{Complexity Analysis}.
Our method achieves substantial reductions in both computational and memory complexity compared to standard MLA.
By compressing the long-range context into $m$ representative key–value pairs and retaining only $w$ recent tokens, \sexyname reduces the attention computation from $\mathcal{O}(L^2)$ to $\mathcal{O}(L(m+w)) = \mathcal{O}(Lm)$, where $m \ll L$ and $w$ is a small constant. 
Correspondingly, the KV cache memory footprint is reduced from $\mathcal{O}(L)$ to $\mathcal{O}(m+w)=\mathcal{O}(m)$, enabling efficient handling of long contexts with substantially lower computational and memory costs. 
We provide a detailed derivation in Appendix~\ref{app:complexity}.

\section{Experiments}
\subsection{Experimental Setup}
\textbf{Evaluation Benchmarks}.
We evaluate our method on both long-context and short-context tasks to comprehensively assess its performance. 
For long-context evaluation, we adopt \textbf{LongBench-E}~\cite{bai2023longbench} and \textbf{RULER}~\cite{hsieh2024ruler}, which comprehensively test retrieval, reasoning, and aggregation over sequences up to 128K tokens. 
To ensure no degradation on standard tasks, we further evaluate on three established short-context benchmarks: 
\textbf{MMLU}~\cite{hendrycks2021measuring}, \textbf{GSM8K}~\cite{cobbe2021gsm8k}, and \textbf{MBPP}~\cite{austin2021program}. 
We put more details in Appendix~\ref{app:benchmarks}.

\begin{table}[t]
    \centering
    \caption{Comparison of efficient attention methods.}
    \label{tab:comparisons}
    \small
    \renewcommand{\tabcolsep}{4pt}
    \begin{tabular}{c|ccc} 
        \hline
        Methods & Computation & Memory & No Extra Params \\ 
        \hline
        Minference & $<\mathcal{O}(L^2)$ & $\mathcal{O}(L)$ & $\checkmark$\\
        FlexPrefill & $<\mathcal{O}(L^2)$ & $\mathcal{O}(L)$ & $\checkmark$ \\
        KDA & $<\mathcal{O}(L^2)$ & $\mathcal{O}(L)$ &  \\
        Ours & $\mathcal{O}(Lm)$ & $\mathcal{O}(m)$ & $\checkmark$ \\
        \hline
        \end{tabular}
\end{table}

\begin{table*}[t]
\centering
\small
\caption{Comparisons on LongBench-E~\cite{bai2023longbench}. We report the latency of DeepSeek V2-Lite within contexts of 64K on H200 GPUs.}
\label{tab:longbench}
\begin{tabular}{l|cccccc|c|c}
\hline
Methods & S. QA & M. QA & Sum. & F. S. & Syn. & Code & Avg.↑ & FTL (s) ↓
\\ \hline

DeepSeek V2-Lite~\textit{\small{(MLA)}} & \textbf{23.92} & \underline{11.15} & \underline{17.52} & \textbf{63.19} & \underline{3.09} & \underline{58.17} & \textbf{29.51} 
& 3.20
\\
~~$\bullet$~MInference~\cite{jiangminference} & 14.54 & 5.61 & 13.80 & 41.15 & 2.17 & 41.01 & 19.71 
& 2.99 (1.1$\times$)
\\
~~$\bullet$~FlexPrefill~\cite{laiflexprefill} & 14.35 & 5.94 & 14.10 & 48.00 & 2.22 & 41.72 & 21.05 
& 2.51 (1.3$\times$)
\\
~~$\bullet$~KDA~\cite{kimiteam2025kimilinear} & 18.81 & \textbf{11.68} & 14.22 & 56.97 & \textbf{3.25} & 57.97 & 27.15 
& 2.47 (1.3$\times$)
\\
\rowcolor{pink!25} ~~$\bullet$~\sexyname (Ours) & \underline{22.61} & 8.90 & \textbf{23.74} & \underline{58.27} & 2.33 & \textbf{58.67} & \underline{29.09} 
& \textbf{1.80 (1.8$\times$)}
\\ \hline
\end{tabular}
\end{table*}

\begin{table*}[t]
\small
\centering
\caption{Comparisons on RULER~\cite{hsieh2024ruler} across 4-128K context. We report the latency of DeepSeek V2-Lite within contexts of 128K on H200 GPUs.}\label{tab:ruler}
\renewcommand{\tabcolsep}{6pt}
\begin{tabular}{l|cccccc|c|c} 
\hline
\multicolumn{1}{c|}{Methods} & 4K & 8K    & 16K   & 32K   & 64K  & 128K & Avg.↑ & FTL (s) ↓\\ 
\hline
DeepSeek V2-Lite~(\textit{\small{MLA}})                  & \textbf{79.51} & \textbf{78.82} & 62.28 & \underline{58.33} & \underline{47.57} & \underline{23.96} & \textbf{58.91} & 10.78 \\
~~$\bullet$~MInference~\cite{jiangminference}                  & 76.88 & 69.69 & 44.06 & 20.31 & 10.31 & 4.34 & 37.60 & 5.66 (1.9$\times$)\\
~~$\bullet$~FlexPrefill~\cite{laiflexprefill}                  & 72.81 & 70.00 & 56.25 & 19.06 & 9.38 & 7.19 & 39.11 & 5.38 (2.0$\times$)\\
~~$\bullet$~KDA~\cite{kimiteam2025kimilinear}                 & 72.22 & 73.61 & \textbf{64.24} & 51.39 & 44.10 & 22.22 & 54.63 & 4.96 (2.2$\times$)\\
\rowcolor{pink!25} ~~$\bullet$~\sexyname~(Ours) & \underline{77.19} & \underline{77.19} & \underline{63.75} & \textbf{59.69} & \textbf{50.63} & \textbf{24.38} & \underline{58.80} & \textbf{4.40 (2.5$\times$)} \\
\hline
\end{tabular}
\end{table*}

\begin{figure*}[h!]
    \centering
    \includegraphics[width=0.9\linewidth]{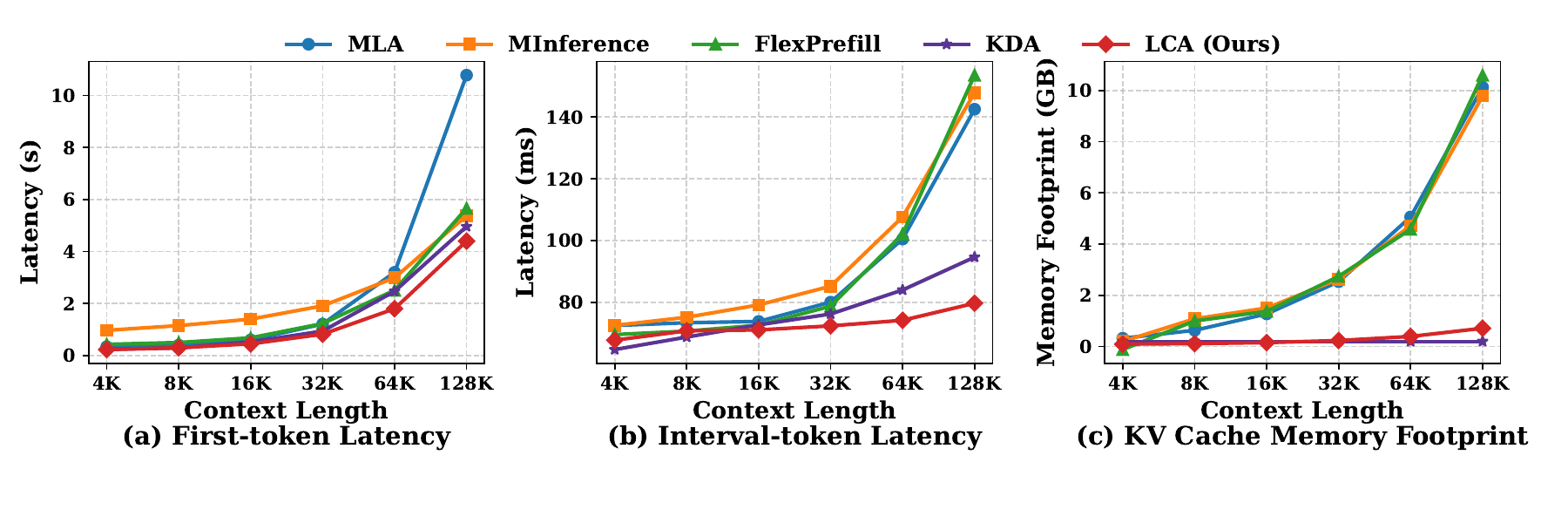}
    \vspace{-15pt}
\caption{Efficiency comparison of different context lengths. (a) presents the first-token generation latency. (b) shows the per-token (interval) latency during decoding. (c) illustrates the GPU memory footprint of KV cache.}~\label{fig:speed_memory}
\end{figure*}

\noindent\textbf{Implementation Details}. 
We replace the MLA layers of DeepSeek-V2-Lite (16B)~\cite{deepseekai2024deepseekv2} with our \sexyname, adopting it as the first publicly released MLA-based model in our study. This choice is motivated by its moderate model size, which enables comprehensive experimentation under limited GPU resources.
We develop a highly optimized kernel of our \sexyname using Triton~\cite{tillet2019triton} for high efficiency.
We fine-tune the model for 1,000 steps on the SlimPajama dataset (32K length)~\cite{fu2024data}. 
The group size $g$ and the window size $w$ are set to 16 and 1024, respectively. 
All experiments are conducted on 8$\times$H200 GPUs. 
Further implementation specifics are detailed in Appendix~\ref{app:implementation}.

\noindent\textbf{Compared Methods}. 
We evaluate \sexyname\ against three representative efficient attention methods: 
two dynamic sparse attention approaches (\textbf{FlexPrefill}~\cite{laiflexprefill} and \textbf{MInference}~\cite{jiangminference}) 
and a gated linear attention variant (\textbf{KDA}~\cite{kimiteam2025kimilinear}).
Since FlexPrefill and MInference are designed for standard self-attention, we adapt them to operate on reconstructed KV matrices from MLA's latents. 
For KDA, which requires integration during training from scratch, we implement a controlled adaptation: we augment the pretrained model with additional modules required by KDA and fine-tune it under the same conditions as \sexyname. 
Details are provided in Appendix~\ref{app:baselines}. We summarize the differences from existing methods in Table~\ref{tab:comparisons}.

\begin{table*}[tb]
\centering
\begin{minipage}[t]{0.48\textwidth}
\centering
\small
\captionof{table}{Comparisons on short-context tasks.}\label{tab:short_tasks}
\renewcommand{\tabcolsep}{2pt}
\begin{tabular}{l|ccc}
\hline
Methods & MMLU & GSM-8K & MBPP \\ \hline
DeepSeek V2-Lite~(\textit{\small{MLA}}) & \textbf{57.12} & \textbf{41.47} & \textbf{54.09}  \\
~~$\bullet~$FlexPrefill~\cite{laiflexprefill} & 53.85 & 33.74 & 52.14 \\
~~$\bullet~$MInference~\cite{jiangminference} & 51.01 & 32.90 & 46.69 \\
~~$\bullet~$KDA~\cite{kimiteam2025kimilinear} & 56.31 & 37.45 & 50.97 \\
\rowcolor{pink!25}~~$\bullet~$\sexyname (Ours) & \underline{57.04} & \underline{41.17} & \underline{53.31} \\ \hline
\end{tabular}
\end{minipage}
\hfill
\begin{minipage}[t]{0.48\textwidth}
\centering
\small
\captionof{table}{Ablation study on the pooling strategies for $\mathbf{c}^{\text{rep}}$ and $\mathbf{k}^{R_{\text{rep}}}$. ``MaxSel'' denotes max selection.}
\label{tab:ablation_pooling}
\renewcommand{\tabcolsep}{2pt}
\vspace{-3pt}
\begin{tabular}{l|cccc}
\hline
$\mathbf{c}_{\text{rep}}$ $\backslash$ $\mathbf{k}^{R_{\text{rep}}}$ & MaxPool & MeanPool & WeightedPool & MaxSel \\
\hline
MaxPool      & 27.44 & 27.01 & 27.43 & 28.89 \\
MeanPool     & 27.38 & 26.39 & 27.67          & 28.84 \\
WeightedPool & 27.93 & 27.04 & 27.83          & \textbf{29.09} \\
MaxSel       & 26.94 & 27.65 & 27.79          & 28.93 \\

\hline
\end{tabular}
\end{minipage}
\end{table*}

\begin{table*}[tb]
\centering
\begin{minipage}[t]{0.48\textwidth}
\small
\centering
\captionof{table}{\revised{Performance on the LongBench-E and efficiency at 128K context with MiniCPM3-4B.}}\label{tab:minicpm3}
\renewcommand{\tabcolsep}{1pt}
\resizebox{0.96\textwidth}{!}{
\begin{tabular}{l|c|cc}
\hline
Methods & Avg. $\uparrow$& FTL (s) $\downarrow$ & KV cache (GB) $\downarrow$ \\
\hline
MiniCPM3-4B (\textit{\small{MLA}}) & 43.28 & 8.05 & 4.36 \\
~~$\bullet$~\sexyname (Ours) & 42.05 & \textbf{3.61 (2.2$\times$)} & \textbf{0.30 (93\%$\downarrow$)} \\
\hline
\end{tabular}
}
\end{minipage}
\hfill
\begin{minipage}[t]{0.48\textwidth}
\small
\centering
\captionof{table}{\revised{Performance on the OlympiadBenchMath and efficiency at 128K context with Distill-Qwen-7B.}}\label{tab:gqa_model}
\renewcommand{\tabcolsep}{1pt}
\resizebox{0.97\textwidth}{!}{
\begin{tabular}{l|c|cc}
\hline
Methods & Acc. $\uparrow$& FTL (s) $\downarrow$& KV cache (GB) $\downarrow$ \\
\hline
Distill-Qwen-7B (\textit{\small{GQA}}) & 50.85 & 9.87 & 7.00 \\
~~$\bullet$~\sexyname (Ours) & 50.13 & \textbf{3.04 (3.25$\times$)} & \textbf{0.49 (93\%$\downarrow$)} \\
\hline
\end{tabular}
}
\end{minipage}
\end{table*}

\begin{figure*}[th]
\begin{minipage}[t]{0.41\textwidth}
\centering
\includegraphics[width=\linewidth]{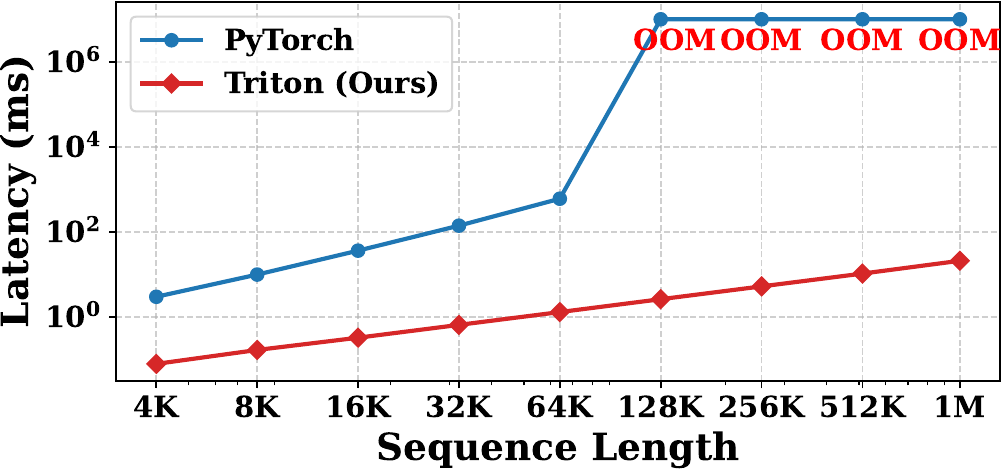}
\caption{Latency comparison between our Triton kernel and the PyTorch implementation.}
\label{fig:latency_triton}
\end{minipage}
\hfill
\begin{minipage}[t]{0.27\textwidth}
\centering
\includegraphics[width=\linewidth]{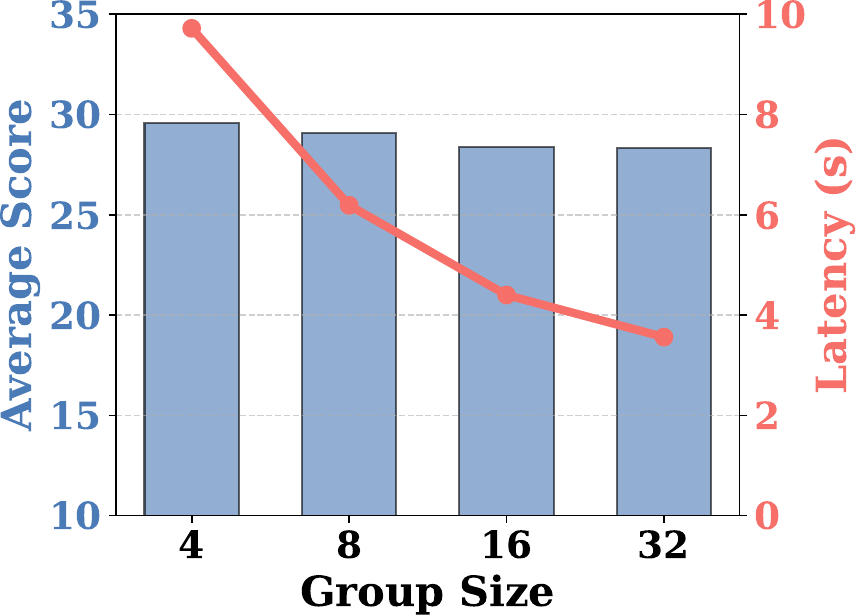}
\caption{Ablations on $g$.}
\label{fig:ablation_group_size}
\end{minipage}
\hfill 
\begin{minipage}[t]{0.27\textwidth}
\centering
\includegraphics[width=\linewidth]{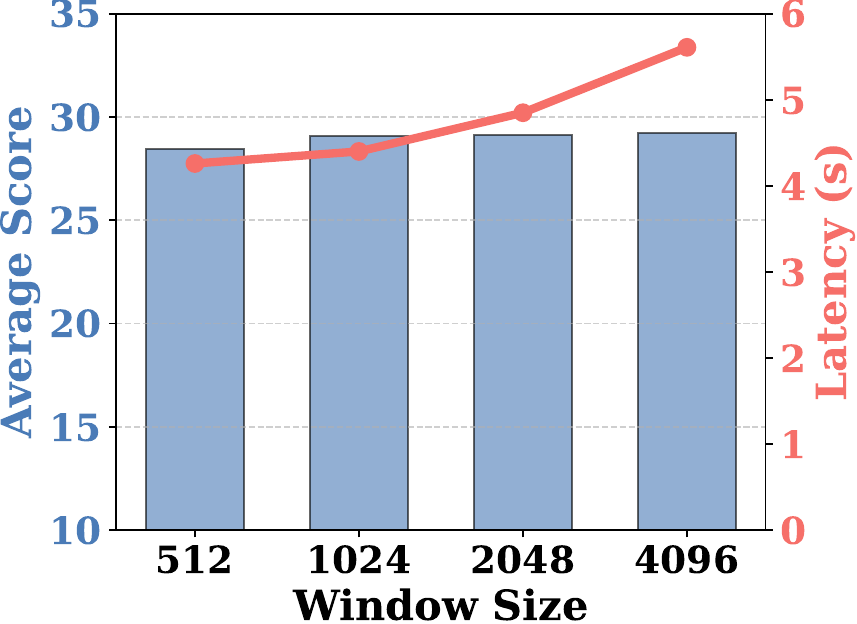}
\caption{Ablations on $w$.}
\label{fig:ablation_window_size}
\end{minipage}
\end{figure*}

\begin{table*}[t]
\centering
\small
\caption{Ablation on the number of queries used to compute importance scores. Default setting is 16.}
\begin{tabular}{c|cccccc|c}
\hline
\#Queries & S. QA & M. QA & Sum. & F. S. & Syn. & Code & Avg. \\
\hline
32 & 23.01 & 8.82 & 23.87 & 58.33 & 2.41 & 59.09 & 29.26 \\
16 (default) & 22.61 & 8.90 & 23.74 & 58.27 & 2.33 & 58.67 & 29.09 \\
8 & 22.34 & 8.37 & 23.17 & 56.22 & 2.09 & 58.41 & 28.43 \\
\hline
\end{tabular}
\label{tab:ablation_num_query}
\end{table*}

\subsection{Performance Comparisons}
\textbf{Comparisons on Long Context Modeling}. In Table~\ref{tab:longbench}, our \sexyname achieves an average score of 29.09 on LongBench-E, nearly matching the 29.51 of the original MLA while achieving a \textbf{1.8$\times$} speedup at 64K context. Notably, \sexyname outperforms MLA on the code generation (58.67 vs. 58.17), demonstrating its capability to preserve critical information. In contrast, existing sparse attention methods suffer significant performance degradation, with Minference and FlexPrefill dropping to 19.71 and 21.05, respectively, due to their sparsification strategy that may discard essential information.

In Table~\ref{tab:ruler}, \sexyname maintains strong performance across all lengths on RULER, achieving the highest scores at 32K, 64K, and 128K contexts. At 128K, \sexyname attains 24.38, surpassing MLA's 23.96 while reducing latency by \textbf{2.5$\times$}. This demonstrates that our latent-space condensation effectively preserves long-range dependencies even at extreme lengths. Notably, while other methods show severe performance degradation at longer contexts (\textit{e.g.}, Minference dropping to 4.34 at 128K), \sexyname maintains stable performance, highlighting the advantage of direct latent-space condensation over reconstruction-based sparsification.

\noindent\textbf{Comparisons on Short‑Context Tasks}.
To verify the gains do not compromise performance on standard tasks, we evaluate \sexyname on three widely-used short-context benchmarks. In Table~\ref{tab:short_tasks}, \sexyname achieves nearly identical performance to the original MLA across all three benchmarks. This demonstrates that the latent-space condensation introduced by \sexyname does not harm the model's capability on short-context tasks.

\subsection{Comparisons on Computational Efficiency}
In Figure~\ref{fig:speed_memory}, \sexyname achieves consistent and scalable efficiency gains across various context lengths. It reduces prefilling latency from short to long contexts, achieving \textbf{2.5$\times$} speedup over MLA at 128K, while Minference and FlexPrefill only accelerate prefilling beyond 64K. During decoding, \sexyname maintains low per-token latency (reducing by \textbf{1.8$\times$} at 128K) and reduces KV cache by \textbf{90\%} (from 10.13 GB to 0.71 GB), whereas Minference and FlexPrefill cannot reduce the KV cache. KDA suffers severe performance degradation (Tables~\ref{tab:longbench} and \ref{tab:ruler}) despite maintaining a small cache, \sexyname preserves competitive performance alongside its efficiency gains. 

In Figures~\ref{fig:latency_triton}, we compare the latency of our custom Triton kernel with the PyTorch implementation. Our kernel achieves consistently lower latency across all context lengths, with up to \textbf{24.4$\times$} speedup at 64K. While the PyTorch implementation encounters OOM beyond 64K, our kernel scales efficiently to \textbf{1M tokens} with feasible latency and linear memory growth. This optimization ensures that the efficiency advantages of \sexyname are fully realized in practice.

\subsection{Scalability and Generalization}\label{sec:generality}
\revised{We evaluate \sexyname's broader applicability across different model scales and attention variants.}

\revised{\textbf{Generalization to Other Model Scales}. We apply \sexyname to MiniCPM3-4B~\cite{hu2024minicpm}, another publicly available MLA-based model. As shown in Table~\ref{tab:minicpm3}, \sexyname achieves a 2.2× prefilling speedup and 93\% KV cache reduction at 128K context, while retaining competitive performance on LongBench‑E. These results demonstrate that \sexyname remains effective across different model scales. The consistent behavior across two models of different sizes suggests that the condensation mechanism primarily depends on latent sequence redundancy.}

\revised{\textbf{Extension to Other Attention Variants}. \sexyname is fundamentally architecture-agnostic. We adapt its dual-path condensation strategy to standard grouped-query attention (GQA) by applying max selection to keys (preserving RoPE positional information) and weighted pooling to values. We evaluate this adaptation on DeepSeek-R1-Distill-Qwen-7B~\cite{guo2025deepseekr1}, a representative thinking model that employs GQA, on the OlympiadBenchMath benchmark. As shown in Table~\ref{tab:gqa_model}, \sexyname delivers comparable accuracy while achieving a 3.25× inference speedup and 93\% KV cache reduction at 128K context. This confirms that \sexyname extends effectively  to other attention variants.}

\subsection{Ablations}
All ablations are evaluated on the LongBench-E benchmark, reporting both the average score and the first-token latency at 128K context length.

\noindent\textbf{Effect on Pooling Strategies}. We investigate the pooling strategies for generating the representative latent $\mathbf{c}^{\text{rep}}$ and positional anchor $\mathbf{k}^{R_{\text{rep}}}$. In Table~\ref{tab:ablation_pooling}, the combination of \textbf{attention-weighted pooling} for $\mathbf{c}^{\text{rep}}$ and \textbf{max selection} for $\mathbf{k}^{R_{\text{rep}}}$ achieves the highest score. This validates our design: weighted pooling preserves semantic information from all tokens in the group, while max selection avoids distorting the nonlinear positional encoding.

\noindent\textbf{Effect on Group Size.} We investigate the impact of the group size $g$ on model performance and efficiency by varying $g$ in $\{4, 8, 16, 32\}$ while keeping the local window size fixed at $w = 1024$. As shown in Figure~\ref{fig:ablation_group_size}, smaller group sizes yield higher average scores on LongBench-E. This improvement stems from finer-grained condensation, which preserves more detailed semantic information within each group. However, this comes at the cost of increased computational latency.
To balance efficiency and performance, we choose $g = 16$ as the default setting for all subsequent experiments, which provides a favorable trade-off.

\noindent\textbf{Effect on Window Size.} We evaluate the influence of the local window size $w$ by varying it across $\{512, 1024, 2048, 4096\}$, with the group size fixed at $g = 16$. Larger windows preserve more local context and improve accuracy, yet also increase latency, as shown in Figure~\ref{fig:ablation_window_size}. We select $w = 1024$ as the default, offering a good compromise between preserving fine-grained local dependencies and maintaining inference efficiency.

\revised{\noindent\textbf{Effect on the Number of Summary Queries.} We ablate the number of queries used to compute the summary query \(\bar{\mathbf{q}}\) (Eq. (6)) while fixing \(g=16\) (Table~\ref{tab:ablation_num_query}). Using 32 queries slightly improves the average score from 29.09 to 29.26, whereas using 8 queries degrades it to 28.43. The default choice of \(g=16\) thus offers a robust trade-off, as the last \(g\) tokens already capture the immediate decoding focus with little gain from a larger window.}

\section{Conclusion}
In this paper, 
we propose \sexyname, an efficient attention mechanism that performs structured context condensation directly within MLA's latent space. \sexyname explicitly handles the disentangled semantic and positional components in MLA: aggregating semantic latents via query-aware pooling while preserving precise positional information through anchor selection. This approach jointly reduces KV cache size and attention complexity without introducing additional parameters.
Theoretically, we prove that the approximation error of our method is uniformly bounded, independent of context length. Extensive experiments validate the effectiveness of \sexyname, demonstrating significant efficiency gains (up to \textbf{2.5$\times$} prefilling speedup and \textbf{90\%} KV cache reduction at 128K context) while maintaining competitive model performance. 

\section*{Limitations}
While \sexyname provides a general design, its implementation requires a custom kernel to achieve promising efficiency, which may involve additional engineering effort for deployment in new frameworks. The current implementation of \sexyname is optimized for common bfloat16/float16 precision settings, and its kernel-level optimizations have not been extensively explored for lower-precision formats (\textit{e.g.}, int8 quantization). 
These are typical engineering and evaluation scoping considerations that do not affect the core algorithmic contribution, and they point to straightforward extensions in future work.

\revised{Beyond engineering considerations, \sexyname exhibits a modest accuracy degradation in tasks demanding precise token-level retrieval under aggressive condensation.
As analyzed in Appendix~\ref{app:failure_case}, multi-document question answering scenarios where attention concentrates sharply on few tokens can attenuate critical signals. This trade-off is inherent to condensation-based methods and highlights the importance of matching the compression strength to the deployment context.}

\section*{Information About Use Of AI Assistants}
During the preparation of this work, we used large language models (LLMs) only for text polishing and grammar correction. All central research ideas, methodological designs, experimental implementations, and result interpretations are the original contributions of the authors. The AI assistants were not involved in any core intellectual content.

\section*{Acknowledgments}
This work was supported in part by the Joint Funds of the National Natural Science Foundation of China under Grant No.U24A20327, the Guangdong S\&T Program under Grant No.2026B0101110001, the GuangDong Basic and Applied Basic Research Foundation under Grant No.2026A1515010388, the Postdoctoral Fellowship Program of CPSF under Grant GZC20251043, and the project of 62536003, PCL2025A14.

\bibliography{example_paper}
\clearpage

\appendix

In the supplementary, we provide a detailed proof, more details on the \sexyname method.
We organize our supplementary as follows.

In Section~\ref{sec:theory}, we provide the detailed proof for Proposition~\ref{lem:optimal-pooling} and Theorem~\ref{thm:uniform-bound}. 

In Section~\ref{app:complexity}, we analysis the computation and memory complexity of \sexyname in detail.

In Section~\ref{app:exp_details}, we provide more experimental details of \sexyname. 

\section{Theoretical Analysis}
\label{sec:theory}

\subsection{Proof of Proposition~\ref{lem:optimal-pooling}}
\label{app:proof-pooling}

We provide a complete and self-contained proof of Proposition~\ref{lem:optimal-pooling}, which establishes the optimality of attention-weighted pooling for latent-space condensation.

\noindent \textbf{Proposition \ref{lem:optimal-pooling}} (\textbf{Optimal Condensation via Weighted Pooling})
\emph{Let \(\{\mathbf{c}_i \in \mathbb{R}^{d_c}\}_{i=1}^g\) be a set of latent vectors and \(\boldsymbol{\alpha} \in \mathbb{R}^{g}\) a probability distribution over them, with \(\alpha_i \ge 0\) and \(\sum_i \alpha_i = 1\). 
Consider the expected squared reconstruction error
\[
\mathcal{L}(\mathbf{c}^{\text{rep}}) = \mathbb{E}_{i \sim \boldsymbol{\alpha}} \bigl[ \|\mathbf{c}_i - \mathbf{c}^{\text{rep}}\|_2^2 \bigr] = \sum_{i=1}^g \alpha_i \|\mathbf{c}_i - \mathbf{c}^{\text{rep}}\|_2^2 .
\]
The vector \(\mathbf{c}^{\text{rep}}\) that minimizes \(\mathcal{L}\) is uniquely given by the convex combination \(\mathbf{c}^{\text{rep}} = \sum_{i=1}^g \alpha_i \mathbf{c}_i.\)
}

\begin{proof}
First, expand the squared Euclidean norm:
\begin{equation*}
\begin{split}
\|\mathbf{c}_i - \mathbf{c}^{\text{rep}}\|_2^2 &= (\mathbf{c}_i - \mathbf{c}^{\text{rep}})^\top (\mathbf{c}_i - \mathbf{c}^{\text{rep}}) \\ 
&= \mathbf{c}_i^\top \mathbf{c}_i - 2 \mathbf{c}_i^\top \mathbf{c}^{\text{rep}} + \mathbf{c}^{\text{rep}\top} \mathbf{c}^{\text{rep}}.
\end{split}
\end{equation*}
Substituting into the loss function yields
\begin{equation*}
    \begin{split}
\mathcal{L}(\mathbf{c}^{\text{rep}}) \small{=} \sum_{i=1}^g \alpha_i \mathbf{c}_i^\top \mathbf{c}_i &\small{-} 2 \Bigl(\sum_{i=1}^g \alpha_i \mathbf{c}_i\Bigr)^\top \mathbf{c}^{\text{rep}} \\
&\small{+} \Bigl(\sum_{i=1}^g \alpha_i\Bigr) \mathbf{c}^{\text{rep}\top} \mathbf{c}^{\text{rep}}.
    \end{split}
\end{equation*}
Since \(\sum_i \alpha_i = 1\), the last term simplifies to \(\mathbf{c}^{\text{rep}\top} \mathbf{c}^{\text{rep}}\). Define the weighted mean \(\bar{\mathbf{c}} = \sum_{i=1}^g \alpha_i \mathbf{c}_i\). Then, we have
\[
\mathcal{L}(\mathbf{c}^{\text{rep}}) = \sum_{i=1}^g \alpha_i \mathbf{c}_i^\top \mathbf{c}_i \;-\; 2 \bar{\mathbf{c}}^\top \mathbf{c}^{\text{rep}} \;+\; \mathbf{c}^{\text{rep}\top} \mathbf{c}^{\text{rep}}.
\]

To find the minimizer, we compute the gradient with respect to \(\mathbf{c}^{\text{rep}}\):
\[
\nabla_{\mathbf{c}^{\text{rep}}} \mathcal{L} = -2 \bar{\mathbf{c}} + 2 \mathbf{c}^{\text{rep}}.
\]
Setting the gradient to zero gives the necessary condition:
\[
\mathbf{c}^{\text{rep}} = \bar{\mathbf{c}} = \sum_{i=1}^g \alpha_i \mathbf{c}_i.
\]

To confirm that this stationary point is indeed a minimum, we examine the Hessian matrix:
\[
\nabla_{\mathbf{c}^{\text{rep}}}^2 \mathcal{L} = 2 \mathbf{I}_{d_c},
\]
where \(\mathbf{I}_{d_c}\) is the \(d_c \times d_c\) identity matrix. This Hessian is positive definite for all \(\mathbf{c}^{\text{rep}}\), implying that the function \(\mathcal{L}\) is strictly convex. Thus, the point \(\mathbf{c}^{\text{rep}} {=} \sum_i \alpha_i \mathbf{c}_i\) is the unique global minimizer.
\end{proof}

\begin{remark}
In \sexyname, for a group \(G_j\) of size \(g\), we treat the normalized attention scores \(\alpha_i^{(j)}\) (Eq.~\eqref{eq:importance-score}) as the probability distribution \(\boldsymbol{\alpha}\). The vectors \(\mathbf{c}_i^{KV}\) are the latent representations of the tokens in the group. Applying Proposition~\ref{lem:optimal-pooling} demonstrates that the weighted pooling operation $\mathbf{c}_j^{\text{rep}} {=} \sum_{i \in G_j} \alpha_i^{(j)} \mathbf{c}_i^{KV}$ is the optimal choice for minimizing the expected reconstruction error of the group's latent contributions. This theoretical guarantee ensures that our condensation procedure minimizes the specified weighted reconstruction error of the group, as measured by mean squared error.
\end{remark}

\paragraph{Comparison with Alternative Pooling Strategies.}
Other common pooling strategies include \emph{mean pooling} (\(\mathbf{c}^{\text{rep}} = \frac{1}{g}\sum_i \mathbf{c}_i\)) and \emph{max pooling} (selecting the vector with the highest norm or a heuristic score). Mean pooling is a special case of our weighted pooling where \(\alpha_i = 1/g\) for all \(i\), which implicitly assumes all tokens are equally important—an assumption that rarely holds in natural language. Max pooling, while preserving the most salient single vector, discards information from other tokens entirely and can be unstable. Our attention-weighted pooling generalizes mean pooling by incorporating token‑specific importance, and it avoids the information‑discarding drawback of max pooling. Proposition~\ref{lem:optimal-pooling} shows that when the importance scores \(\alpha_i\) accurately reflect relevance, weighted pooling is objectively optimal in the least‑squares sense.

\subsection{Proof of Theorem~\ref{thm:uniform-bound}}\label{app:thm_proof}

We theoretically analyze the uniform error bound for \sexyname in Theorem~\ref{thm:uniform-bound}. The argument closely follows the structure of the original theorem but employs the notation introduced in Section~\ref{sec:theory} that reflects the actual grouping and condensation performed by \sexyname.

\textbf{Notations.}
For a query at position $t$, partition the preceding tokens into the recent window \thmrevise{$R=\{t-k+1,\ldots,t\}$} and the distant history \thmrevise{$H=\{1,\ldots,t-k\}$}, \thmrevise{where $k=w+((t-w)\bmod g)$ for $t\ge w+g$. For shorter sequences, LCA reverts to standard MLA.} The distant history is further divided into $m$ contiguous groups $G_1,\ldots,G_m$, each of size $g$ \thmrevise{, with any remaining tokens included in the recent window}.

The original MLA attention at position $t$ is
\begin{equation*}
    \mathbf{Attn}_t = \sum_{i=1}^{t} p_i \mathbf{v}_i,~~~~~ p_i = \frac{e^{a_i}}{Z},
\end{equation*}
where $a_i = \frac{\mathbf{q}_t^\top\mathbf{k}_i}{\sqrt{d_k}},\; Z=\sum_{j=1}^{t} e^{a_j}.$

In \sexyname, the same query is compared with the representative keys for distant tokens and with the original keys for recent tokens. Hence the attention weights are defined by
\[
\mathbf{Attn}_t^{\mathrm{LCA}} = \sum_{i\in R} p'_i \mathbf{v}_i \;+\; \sum_{j=1}^{m} \beta_j \mathbf{v}_j^{\mathrm{rep}},
\]
where the logits for recent tokens are unchanged, \ie,
$b_i = a_i \;(i\in R)$,
and for a distant group $G_j$ we use the single representative logit $b_j^{\mathrm{rep}} = \frac{\mathbf{q}_t^\top\mathbf{k}_j^{\mathrm{rep}}}{\sqrt{d_k}}$.
The corresponding weights are
\begin{align*}
    &p'_i = \frac{e^{b_i}}{Z'}\;(i\in R),\qquad 
\beta_j = \frac{e^{b_j^{\mathrm{rep}}}}{Z'} \\ 
&Z' = \sum_{i\in R} e^{b_i} + \sum_{j=1}^{m} e^{b_j^{\mathrm{rep}}}.
\end{align*}

To compare the two outputs directly, it is convenient to introduce for every distant token $i\in H$ via the {group‑assigned} key and value:
\[
\tilde{\mathbf{k}}_i = \mathbf{k}_{G(i)}^{\mathrm{rep}},\qquad 
\tilde{\mathbf{v}}_i = \mathbf{v}_{G(i)}^{\mathrm{rep}}.
\]
For recent tokens $i\in R$ we simply set $\tilde{\mathbf{k}}_i = \mathbf{k}_i$, $\tilde{\mathbf{v}}_i = \mathbf{v}_i$. With this convention we can write the \sexyname output compactly as
\[
\mathbf{Attn}_t^{\mathrm{LCA}} = \sum_{i=1}^{t} p'_i \tilde{\mathbf{v}}_i,
\]
where $p'_i$ for $i\in H$ is defined as $p'_i = \beta_{G(i)}/g$ (so that $\sum_{i\in G_j} p'_i = \beta_j$). The deviation hypotheses of the theorem are then 
\[
\|\mathbf{k}_i - \tilde{\mathbf{k}}_i\|_2 \le \delta_k,~ 
\|\mathbf{v}_i - \tilde{\mathbf{v}}_i\|_2 \le \delta_v, ~(i=1,\dots,t).
\]

\noindent \textbf{Theorem \ref{thm:uniform-bound}}
(\textbf{Uniform Error Bound})
\emph{
Fix a query $\mathbf{q}_t$ with $\|\mathbf{q}_t\|_2 \le Q$ and assume $\|\mathbf{v}_i\|_2 \le V$ for all values. For each distant-history group $G_j$ \thmrevise{with group size $g$}, let $\mathbf{k}_j^{\mathrm{rep}},\mathbf{v}_j^{\mathrm{rep}}$
be the representative key and value
{corresponding to the tokens in $G_j$},
and let $\delta_k,\delta_v$ be uniform bounds on the per-token key and value deviations: $\|\mathbf{k}_i - \mathbf{k}_j^{\mathrm{rep}}\|_2 \le \delta_k,
\|\mathbf{v}_i - \mathbf{v}_j^{\mathrm{rep}}\|_2 \le \delta_v.$
\thmrevise{Any remaining tokens are left uncompressed.}
Then the $\ell_2$ error between the full MLA attention and \sexyname satisfies
\[
\begin{aligned}
\big\|\mathbf{Attn}_t-\mathbf{Attn}_t^{\mathrm{LCA}}\big\|_2
&\le V\big(e^{2Q\delta_k/\sqrt{d_k}}-1\big)\\
&\quad+\delta_v
\thmrevise{+2V\,\frac{\sqrt{g}-1}{\sqrt{g}+1}}.
\end{aligned}
\]
}

\begin{proof}
Define the total error $E := \mathbf{Attn}_t - \mathbf{Attn}_t^{\mathrm{LCA}}$. A straightforward algebraic manipulation gives the decomposition:
\begin{equation*}
E = \underbrace{\sum_{i=1}^{t} (p_i - p'_i)\mathbf{v}_i}_{\text{(I)}} \;+\; \underbrace{\sum_{i=1}^{t} p'_i(\mathbf{v}_i - \tilde{\mathbf{v}}_i)}_{\text{(II)}}.
\end{equation*}
We next bound the norms of (I) and (II) separately.

1) We start by bounding the first term. First, we bound the logit perturbations. Using the Cauchy--Schwarz inequality and the given bounds on the query norm and key approximations,
\begin{align*}
    |a_i - b_i| &= \frac{|\mathbf{q}_t^\top(\mathbf{k}_i - \tilde{\mathbf{k}}_i)|}{\sqrt{d_k}}
\le \frac{\|\mathbf{q}_t\|_2 \|\mathbf{k}_i - \tilde{\mathbf{k}}_i\|_2}{\sqrt{d_k}} \nonumber\\
&\le \frac{Q\delta_k}{\sqrt{d_k}} =: \varepsilon,
\end{align*}
\thmrevise{where $b_i$ denotes the representative logit assigned to original position $i$ for analysis (i.e., $b_i=a_i$ for $i\in R$ and $b_i=b_{G(i)}^{\mathrm{rep}}$ for $i\in H$). In the actual \sexyname softmax, each condensed group contributes only one logit.}
Hence each logit changes by at most $\varepsilon$.

Thus, we have $e^{-\varepsilon}e^{b_i} \le e^{a_i} \le e^{\varepsilon} e^{b_i}$ for every $i$. Summing these inequalities over all $i$ yields
\[
e^{-\varepsilon} \sum_{j=1}^{t} e^{b_j} \;\le\; \sum_{j=1}^{t} e^{a_j} \;\le\; e^{\varepsilon} \sum_{j=1}^{t} e^{b_j}.
\]
\thmrevise{Denote $Z=\sum_j e^{a_j}$ and $\bar Z=\sum_j e^{b_j}$, and define $\bar p_i=e^{b_i}/\bar Z$. Here $\bar Z$ counts each representative logit $g$ times, unlike the actual \sexyname normalizer $Z'$, which counts each group once.}
The previous chain implies \thmrevise{$e^{-\varepsilon}\bar Z \le Z \le e^{\varepsilon}\bar Z$}. Combining this with the pointwise bound $e^{-\varepsilon}e^{b_i} \le e^{a_i}$ gives
\[
p_i = \frac{e^{a_i}}{Z} \ge
\thmrevise{\frac{e^{-\varepsilon}e^{b_i}}{e^{\varepsilon}\bar Z}
= e^{-2\varepsilon}\bar p_i}.
\]
Similarly, \thmrevise{$p_i \le e^{2\varepsilon}\bar p_i$}. Therefore,
\begin{equation}\label{A3}
    \thmrevise{e^{-2\varepsilon}\bar p_i \le p_i \le e^{2\varepsilon}\bar p_i}
    \qquad\text{for all }i.
\end{equation}

\thmrevise{From Eqn. (\ref{A3}), we obtain
$|p_i-\bar p_i|\le(e^{2\varepsilon}-1)\bar p_i$.
Summing over $i$ gives
\[
\sum_{i=1}^{t}|p_i-\bar p_i|
\le(e^{2\varepsilon}-1)\sum_{i=1}^{t}\bar p_i
=e^{2\varepsilon}-1.
\]}

\thmrevise{We next bound the difference between $\bar p_i$ and the actual \sexyname weights $p'_i$. Define
\[
A=\sum_{j=1}^{m}e^{b_j^{\mathrm{rep}}},
\qquad
B=\sum_{i\in R}e^{b_i}.
\]
Since all condensed groups have size $g$, we have
$\bar Z=gA+B$, whereas the actual \sexyname normalizer is $Z'=A+B$. The total attention mass assigned to condensed groups is $gA/(gA+B)$ under $\bar p_i$ and $A/(A+B)$ under $p'_i$. The corresponding masses for local tokens are $B/(gA+B)$ and $B/(A+B)$. As the relative weights within each part are unchanged,}
\[
\thmrevise{
\begin{aligned}
\sum_{i=1}^{t}|\bar p_i-p'_i|
&=2\left|\frac{gA}{gA+B}-\frac{A}{A+B}\right|\\
&=\frac{2(g-1)AB}{(gA+B)(A+B)}\\
&\le 2\,\frac{\sqrt{g}-1}{\sqrt{g}+1}.
\end{aligned}
}
\]
\thmrevise{The last inequality follows from
$(\sqrt{g}A-B)^2\ge 0$.}

\thmrevise{By the triangle inequality,}
\begin{equation}\label{A4}
\thmrevise{
\begin{aligned}
\sum_{i=1}^{t}|p_i-p'_i|
&\le \sum_{i=1}^{t}|p_i-\bar p_i|
+\sum_{i=1}^{t}|\bar p_i-p'_i|\\
&\le e^{2\varepsilon}-1
+2\,\frac{\sqrt{g}-1}{\sqrt{g}+1}.
\end{aligned}
}
\end{equation}

Using the uniform bound $\|\mathbf{v}_i\|_2 \le V$ together with Eqn. (\ref{A4}),
\begin{align}\label{A5}
    &\big\| \sum_{i=1}^{t} (p_i - p'_i)\mathbf{v}_i \big\|_2
\le \sum_{i=1}^{t} |p_i - p'_i| \cdot \|\mathbf{v}_i\|_2 \nonumber\\
\le & V \sum_{i=1}^{t} |p_i - p'_i|
\le V\big(e^{2\varepsilon} - 1\big)
\thmrevise{+2V\,\frac{\sqrt{g}-1}{\sqrt{g}+1}}.
\end{align}

2) We bound the second error component. Given
$\|\mathbf{v}_i - \tilde{\mathbf{v}}_i\|_2 \le \delta_v$ and the fact that $p'_i$ form a probability distribution, we have
\begin{align}\label{A6}
    &\big\| \sum_{i=1}^{t} p'_i (\mathbf{v}_i - \tilde{\mathbf{v}}_i) \big\|_2
\le \sum_{i=1}^{t} p'_i \|\mathbf{v}_i - \tilde{\mathbf{v}}_i\|_2 \nonumber\\
\le& \delta_v \sum_{i=1}^{t} p'_i = \delta_v.
\end{align}

3) Combing Eqn. (\ref{A5}) and (\ref{A6}) gives
\[
\|E\|_2 \;\le\; V\big(e^{2\varepsilon} - 1\big)
\;+\; \delta_v
\thmrevise{+2V\,\frac{\sqrt{g}-1}{\sqrt{g}+1}},
\]
where $\varepsilon = Q\delta_k / \sqrt{d_k}$.

\end{proof}

\subsection{Empirical Validation of Key and Value Deviations}\label{app:kv_deviations}
\revised{we empirically measure the per-token key and value deviations $\delta_k$ and $\delta_v$ across different context lengths and tasks. For each setting, we randomly sample 128 sequences and compute the ratio of the L2 norm of the deviation between original keys/values and their condensed group representatives to the L2 norm of the original vectors. The results in Tables~\ref{tab:context_stats} and~\ref{tab:task_stats} show that across all lengths and tasks, the average relative deviation remains below 5\% for keys and 4\% for values, confirming that contiguous tokens within a group are indeed semantically similar and that our condensation introduces only small per-token errors. These observations provide
empirical support for the deviation assumptions in
Theorem~\ref{thm:uniform-bound}.}

\begin{table}[t]
\centering
\small
\caption{\revised{Statistics across different context lengths.}}
\label{tab:context_stats}
\begin{tabular}{c|ccc}
\hline
Context Length & 4K & 32K & 128K \\
\hline
$\|\mathbf{k}_i\|_2$ & 20.00 & 22.13 & 21.39 \\
$\delta_k$ & 0.43 & 0.64 & 0.77 \\
$\delta_k / \|\mathbf{k}_i\|_2$ & 2.15\% & 2.89\% & 3.60\% \\
\hline
$\|\mathbf{v}_i\|_2$ & 20.87 & 21.38 & 21.23 \\
$\delta_v$ & 0.68 & 0.73 & 0.77 \\
$\delta_v / \|\mathbf{v}_i\|_2$ & 3.26\% & 3.41\% & 3.63\% \\
\hline
\end{tabular}
\end{table}

\begin{table}[t]
\centering
\small
\caption{\revised{Statistics across different task types.}}
\label{tab:task_stats}
\begin{tabular}{lccc}
\hline
\textbf{Metric} & \textbf{Multi-Doc.\ QA} & \textbf{Summarization} & \textbf{Code} \\
\hline
$\|\mathbf{k}_i\|_2$ & 22.76 & 21.42 & 20.07 \\
$\delta_k$ & 0.67 & 0.98 & 0.89 \\
$\delta_k / \|\mathbf{k}_i\|_2$ & 2.94\% & 4.58\% & 4.43\% \\
\hline
$\|\mathbf{v}_i\|_2$ & 19.09 & 24.78 & 29.78 \\
$\delta_v$ & 0.56 & 0.51 & 0.98 \\
$\delta_v / \|\mathbf{v}_i\|_2$ & 2.93\% & 2.06\% & 3.29\% \\
\hline
\end{tabular}
\end{table}

\section{Complexity Analysis}\label{app:complexity}
We analyze the computational and memory complexity of \sexyname against standard MLA. Recall that $L$ is the sequence length, $w$ is the short‑range window size, and $m$ is the number of long‑range representatives.

\textbf{Computational complexity.}
The overall computation consists of three main steps. 
First, importance scores are computed between a group‑aware query $\bar{\mathbf{q}}$ and all $L$ keys, requiring $\mathcal{O}(L)$ operations. 
Second, for each of the $m$ groups, we compute normalized attention weights (softmax over $g$ scores) and perform weighted pooling of $g$ latent vectors, which together cost $\mathcal{O}(L)$. 
Finally, attention is computed over the condensed context of size $m + w$ for each of the $L$ queries, resulting in $\mathcal{O}\big(L(m+w)\big)$ operations. 
Since $w$ is a constant, the dominant term becomes $\mathcal{O}(Lm)$, lower than the $\mathcal{O}(L^2)$ complexity of standard MLA.

\textbf{Memory complexity of the KV cache.} In \sexyname, we cache only the $m$ representative latent vectors $\{\mathbf{c}_j^{\text{rep}}\}_{j=1}^m$ and their positional anchors $\{\mathbf{k}_j^{R_{\text{rep}}}\}_{j=1}^m$, along with the $w$ recent tokens stored in full. 
Hence the total cache size scales as $\mathcal{O}(m+w) = \mathcal{O}(m)$.

The reduction in both computation and memory enables \sexyname to handle long contexts with significantly lower hardware requirements while preserving essential information through structured latent‑space condensation.

\section{More Experimental Details}\label{app:exp_details}
\subsection{More Details about Evaluation Benchmarks}~\label{app:benchmarks}
We evaluate our method on both long-context and short-context tasks to comprehensively assess its performance and efficiency. 
For long-context understanding, we use two established benchmarks: 
\textbf{LongBench-E}~\cite{bai2023longbench}, a bilingual suite with uniformly distributed context lengths that measures overall context understanding across 21 diverse tasks; 
and \textbf{RULER}~\cite{hsieh2024ruler}, a synthetic benchmark with 13 subtasks (e.g., retrieval, multi-hop reasoning) designed for fine-grained assessment of long-context capabilities.

Furthermore, to verify that our efficiency gains do not compromise performance on standard tasks, we evaluate on three short-context benchmarks: 
\textbf{MMLU}~\cite{hendrycks2021measuring}, which assesses knowledge breadth and reasoning across 57 academic subjects; 
\textbf{GSM8K}~\cite{cobbe2021gsm8k}, a dataset of grade-school math problems requiring multi-step reasoning; 
and \textbf{MBPP}~\cite{austin2021program}, which evaluates code generation ability through crowd-sourced programming tasks.

\subsection{More Implementation Details}~\label{app:implementation}

We implement and evaluate our proposed \sexyname on the DeepSeek-V2-Lite model (16B parameters)~\cite{deepseekai2024deepseekv2}. We choose this model for two primary reasons: 1) it is the first publicly released model that adopts MLA, providing an ideal foundation for our method which specifically targets MLA's efficiency bottlenecks; 2) its moderate size makes it feasible to conduct extensive experiments and latency measurements within constrained GPU resources. We replace the original MLA layers in DeepSeek-V2-Lite with our \sexyname. To ensure high efficiency, we develop a optimized kernel using Triton~\cite{tillet2019triton}. This custom implementation minimizes memory movement and maximizes hardware utilization during both training and inference. To allow the model to adapt to our dynamic condensation mechanism, we perform continued fine-tuning for only 1,000 steps with 2.0B tokens on the per-source-length upsampled SlimPajama dataset~\cite{fu2024data}. Each training sample has a sequence length of 32K tokens. We use a total batch size of 64, implemented with a micro-batch size of 1 and gradient accumulation steps of 8. Training is conducted using the Deepspeed Zero-3 strategy for efficient memory management. The learning rate is set to $5\times10^{-6}$. We set the group size $g = 16$ and the window size $w = 1024$, respectively. These settings are kept fixed when inference across all experiments. All experiments and latency measurements are performed on a server equipped with 8$\times$H200 GPUs.

\subsection{More Details of Compared Methods}~\label{app:baselines}
To evaluate the effectiveness and efficiency of \sexyname, we compare it against several SoTA efficient attention methods. We benchmark against two recent published methods that perform dynamic sparsification: \textbf{FlexPrefill}~\cite{laiflexprefill} and \textbf{MInference}~\cite{jiangminference}. Note that these methods were originally designed for and evaluated on vanilla self-attention. They cannot be directly applied to MLA's low-rank latent space. For a fair comparison, we adapt these methods to work on the \textit{reconstructed} key and value matrices derived from MLA's latents. This ensures functional equivalence but forfeits the opportunity to exploit the compressed latent structure for further gains. \textbf{FlexPrefill} dynamically selects between a Query-Aware pattern and a Vertical-Slash pattern per attention head, adaptively determining which key-value indices are necessary for computation. We use the official implementation and adhere to the recommended hyperparameters: a score ratio $\gamma=0.9$, a threshold $\tau=0.1$, a minimum retained budget of 1,024 tokens, and a block size of 128. \textbf{MInference} determines optimal static sparse patterns per attention head offline, combined with online dynamic adjustment of computation regions. Following the original paper, we perform offline calibration on the DeepSeek-V2-Lite model to obtain a sparse configuration and use it for online inference.

\begin{table*}[tb]
\centering
\small
\caption{\revised{Effect of group size $g$ with fixed window size $w=1024$.}}
\label{tab:group-size}
\begin{tabular}{c|c c c}
\hline
$g$ & LongBench Avg. & RULER Avg. (4K--128K) & RULER@128K \\
\hline
32 & 28.32 & 58.18 & 23.96 \\
16 & 29.09 & 58.80 & 24.38 \\
8  & 29.38 & 59.42 & 24.63 \\
\hline
\end{tabular}
\end{table*}

\begin{table*}[tb]
\centering
\small
\caption{\revised{Effect of window size $w$ with fixed group size $g=16$.}}
\label{tab:window-size}
\begin{tabular}{c | c c c}
\hline
$w$ & LongBench Avg. & RULER Avg. (4K--128K) & RULER@128K \\
\hline
512  & 28.84 & 57.14 & 23.68 \\
1024 & 29.09 & 58.80 & 24.38 \\
2048 & 29.14 & 59.52 & 25.53 \\
\hline
\end{tabular}
\end{table*}

We also compare against the recently proposed \textbf{Kimi Delta Attention (KDA)}~\cite{kimiteam2025kimilinear}, a gated linear attention variant that introduces a fine-grained diagonalized gate for controlling memory decay and positional awareness.
A critical distinction lies in the training requirements: the original KDA is designed to be integrated into the model during pre‑training from scratch, requiring substantial computational resources and large‑scale data to achieve promsing performance.
In contrast, our \sexyname requires only minimal continued fine‑tuning to recover full performance on a pre‑trained base.
To establish a controlled comparison under constrained resources, we apply the same adaptation protocol to KDA as we do for our method.
Specifically, we take the base DeepSeek‑V2‑Lite model, augment it with the additional parameters required by KDA, apply KDA to the reconstructed KV matrices, and fine‑tune it for 1000 steps on the same per‑source‑length upsampled SlimPajama dataset~\cite{fu2024data} used for \sexyname.

\section{More Discussions}
\subsection{Failure Case Analysis}\label{app:failure_case}
\revised{To understand when \sexyname's condensation may underperform, we analysis per-task performance on LongBench in Table~\ref{tab:longbench}.
On summarization tasks, \sexyname achieves 23.74, substantially outperforming the MLA baseline (17.52), confirming that weighted pooling effectively aggregates diffuse semantic content for gist-based comprehension. In contrast, on multi-document QA, which demands precise retrieval of specific facts across documents, \sexyname scores 8.90 compared to MLA's 11.15, a modest gap that reflects the inherent challenge of preserving concentrated token-level signals under aggressive condensation. The gap narrows on single-document QA (22.61 vs. 23.92), where positional anchoring via max selection helps retain critical information within a coherent document. Notably, on the more challenging RULER benchmark at 128K context, \sexyname achieves 24.38, marginally exceeding MLA's 23.96, demonstrating that the combined effect of weighted pooling, max selection, and the local window preserves retrieval accuracy even under extreme long-context demands. Taken together, these results indicate that \sexyname excels when information is diffusely distributed and incurs only minor degradation in retrieval-heavy scenarios.}

\subsection{Sensitivity Analysis of Group Size and Window Size}
\revised{To better understand how the hyperparameters $g$ (group size) and $w$ (window size) influence performance across context lengths and task types, we conduct additional experiments on LongBench and RULER with varying configurations at inference. Tables~\ref{tab:group-size} and~\ref{tab:window-size} report the results.}

\revised{Smaller $g$ consistently improves accuracy, particularly on retrieval-heavy tasks (RULER@128K improves from $23.96$ to $24.63$). This gain comes at increased computational cost, as analyzed in Figure~\ref{fig:ablation_group_size}.
Larger $w$ yields the most pronounced benefits on very long contexts (RULER@128K rises from $23.68$ to $25.53$), whereas improvements on shorter sequences are marginal.}

\revised{These results confirm that optimal hyperparameter choices depend on both context length and task characteristics. Notably, $g$ and $w$ serve as explicit, interpretable control knobs that allow practitioners to adjust the performance, efficiency trade-off according to deployment constraints. The default configuration ($g=16$, $w=1024$) offers a practical balance between strong efficiency gains and competitive accuracy. A fully adaptive mechanism conditioned on context statistics remains a promising direction for future investigation.}

\subsection{Compatibility with Training from Scratch}
\revised{To verify that \sexyname does not rely on pretrained MLA initialization, we conduct a controlled from-scratch experiment on LLaMA2-7B. Both vanilla self-attention and \sexyname are trained for 1,000 steps under identical optimization settings. \sexyname achieves consistently lower training loss throughout, reaching 6.34 after 1,000 steps compared to 6.98 for the vanilla baseline. This indicates that \sexyname is compatible with training from scratch, imposes no optimization instability, and can even accelerate early-stage convergence without architectural modification beyond attention replacement.}

\end{document}